\documentclass[twoside,11pt]{article}

\usepackage{jmlr2e}

\usepackage{microtype}
\usepackage{graphicx}
\usepackage{caption}
\usepackage{booktabs} 
\usepackage[export]{adjustbox}
\usepackage{subcaption}
\usepackage[dvipsnames]{xcolor}
\usepackage{hyperref}       
\hypersetup{
  breaklinks,
  colorlinks,
  citecolor={green!50!black},
  urlcolor={green!50!black},
  linkcolor={green!50!black}
}
\usepackage{url}            
\usepackage{nicefrac}       
\usepackage{amsmath,amsthm,amssymb}

\usepackage{makecell}
\usepackage{multirow}
\usepackage{multicol}
\usepackage{mathtools}
\usepackage[nameinlink]{cleveref}
\usepackage{setspace}
\usepackage{dsfont}
\newcommand{\N}{\mathds{N}}
\newcommand{\R}{\mathds{R}}

\newtheorem{theorem}{Theorem}
\newtheorem{assumption}{Assumption}
\newtheorem{lemma}{Lemma}
\newtheorem{corollary}{Corollary}
\crefname{assumption}{Assumption}{Assumptions}
\theoremstyle{definition}
\newtheorem{definition}{Definition}[section]

\newcommand{\frsb}{\hat{\partial}}
\newcommand{\dist}{\mathop{\mathrm{dist}}}

\newcommand{\argmin}{\operatornamewithlimits{argmin}}
\newcommand{\prox}{\operatornamewithlimits{prox}}
\DeclarePairedDelimiter\abs{\lvert}{\rvert}

\usepackage[noend]{algorithm2e}
\SetKwComment{Comment}{}{}

\SetCommentSty{mycommentstyle}
\crefname{algocf}{alg.}{algs.}
\Crefname{algocf}{Algorithm}{Algorithms}

\ShortHeadings{A Stochastic Proximal Method for Nonsmooth Regularized Finite Sum Optimization}{Lakhmiri and Orban and Lodi}
\firstpageno{1}

\title{A Stochastic Proximal Method for Nonsmooth Regularized Finite Sum Optimization}

\author{\name Dounia Lakhmiri \email dounia.lakhmiri@polymtl.ca \\
       \addr Department of Mathematics and Industrial engineering\\
       Polytechnique Montreal\\
       Canada Excellence Research Chair in “Data Science\\
       for Real-time Decision-making"\\
       Montreal, Qc, Canada
       \AND
       \name Dominique Orban \email dominique.orban@polymtl.ca \\
       \addr Department of Mathematics and Industrial engineering\\
       Polytechnique Montreal\\
       Montreal, Qc, Canada
       \AND
       \name Andrea Lodi \email andra.lodi@cornell.edu\\
       \addr Canada Excellence Research Chair in “Data Science\\
       for Real-time Decision-making"\\
       Jacobs Technion-Cornell Institute\\
       Cornell University\\
       Ithaca, NY 14850, United States
       }

\editor{}

\begin{document}

\maketitle

\begin{abstract}
We consider the problem of training a deep neural network with nonsmooth regularization to retrieve a sparse and efficient sub-structure. Our regularizer is only assumed to be lower semi-continuous and prox-bounded. We combine an adaptive quadratic regularization approach with proximal stochastic gradient principles to derive a new solver, called SR2, whose convergence
and worst-case complexity are established without knowledge or approximation of the gradient's Lipschitz constant. We formulate a stopping criteria that ensures an appropriate first-order stationarity measure converges to zero under certain conditions. 
We establish a worst-case iteration complexity of $\mathcal{O}(\epsilon^{-2})$ that matches those of related methods like ProxGEN, where the learning rate is assumed to be related to the Lipschitz constant. 
Our experiments on network instances trained on CIFAR-10 and CIFAR-100 with $\ell_1$ and $\ell_0$ regularizations show that SR2 consistently achieves higher sparsity and accuracy than related methods such as ProxGEN and ProxSGD.
\end{abstract}


\section{Introduction}
\label{submission}

We focus on the problem of training neural networks with regularization expressed as 
\begin{equation}%
    \label{eq:main_pb}
    \min_{x} F(x) := f(x) + \mathcal{R}(x),
    \qquad
    f(x) := \frac{1}{N}\sum_{i=1}^N f_i(x),
\end{equation}
where $x \in \R^n$ are the parameters, $f$ is the loss function, and $\mathcal{R}$ may be nonsmooth, nonconvex, and take infinite values. Instances of~\eqref{eq:main_pb} are often used as approximations of
\begin{equation*} 
    \min_{x} \mathbb{E}_{\omega \sim \mathcal{P}}[f(x, \omega)] + \mathcal{R}(x),
\end{equation*}
where $\omega$ follows a distribution $\mathcal{P}$.
In~\eqref{eq:main_pb}, \(\mathcal{R}\) helps select a solution with desirable features among all potential minimizers of \(f\). Examples include the weight decay technique, which uses $\mathcal{R}(x) := \|x\|_2$ to avoid over-fitting the training data \citep{krogh1992simple, zhou2021fixnorm}.
Other applications employ a specific regularizer, whether convex, such as $\|\cdot\|_1$, or nonconvex, such as $\|\cdot\|_0$, to retrieve a sparse sub-network for network pruning \citep{hoefler2021sparsity, wang2019structured, yang2019structured} or quantization \citep{bai2018proxquant, wess2018weighted}. For the rest of this work, we focus on sparsity-promoting $\mathcal{R}$.\\ 

We introduce SR2\footnote{\url{https://github.com/DouniaLakhmiri/SR2}}, a stochastic variant of the quadratic regularization method that solves~\eqref{eq:main_pb} for nonsmooth, nonconvex regularizers.
Our main contributions are
\begin{enumerate}
    \item to the best of our knowledge, the first stochastic adaptive quadratic regularization method for~\eqref{eq:main_pb} with weak assumptions on \(\mathcal{R}\); 

    \item the formulation of a stopping criterion and a first order stationarity measure adapted to nonsmooth, non-convex stochastic optimization problems;
    
    \item the convergence of a first-order stationarity measure to zero without assuming knowledge of the Lipschitz constant of $\nabla f$, and worst-case \(\mathcal{O}(\epsilon^{-2})\) iteration complexity;  
    \item numerical experiments on multiple instances of deep neural networks (DNNs) to retrieve a sparse sub-network. In most cases, SR2 achieves high sparsity levels without post-treatment. A comparison against two related proximal solvers, ProxSGD and ProxGEN, in terms of accuracy and sparsity of the solution is favorable for SR2.
\end{enumerate}

\subsection{Background and related work}

The stochastic gradient (SG) method \citep{kiefer1952stochastic, robbins1951stochastic}, and its variants \citep{ruder2016overview, adam, nguyen2017sarah}, are a common approach for~\eqref{eq:main_pb} when $\mathcal{R} = 0$. At iteration \(t\), SG selects a sample set \(\xi_t \subseteq \{1, \ldots, n\}\), computes the sampled gradient \(g_t = \frac{1}{|\xi_t|} \sum_{i \in \xi_t} \nabla f_i(x_t)\), and updates
\begin{equation}\label{eq:SG_update}
    x_{t+1} \leftarrow x_t -\alpha_t g_t,
\end{equation}
where $\alpha_t > 0$ is the step size, or learning rate. SG and variants typically accept every step regardless of whether the objective decreases or not.
For this reason, we do not refer to it as SGD, where D would stand for \emph{descent}. SG can be shown to converge in expectation under certain assumptions on the learning rate and on the quality of \(g_t\) \citep{Bottou2018}.\\

Proximal gradient descent (PGD) \citep{fukushima-mine-1981} is suited to the structure of~\eqref{eq:main_pb}, i.e., when $\mathcal{R} \not = 0$. At iteration \(t\), it computes a step
\begin{equation}%
    \label{eq:prox_subprob}
    s_t \in \argmin_s \tfrac{1}{2} \|s + \alpha_t g_t\|^2 + \alpha_t \mathcal{R}(x_t + s)
    = \argmin_s g_t^T s + \tfrac{1}{2} \alpha_t^{-1} \|s\|^2 + \mathcal{R}(x_t + s)
    := \prox_{\alpha_t \mathcal{R}(x_t + \cdot)}(- \alpha_t g_t)
\end{equation}
for a prescribed \(\alpha_t > 0\), followed by the update \(x_{t+1} := x_t + s_t\).\\

Observe that due to the nonsmoothness and/or nonconvexity of \(\mathcal{R}\), the right-hand side of~\eqref{eq:prox_subprob} may contain several elements.
The key point is that a closed form solution of~\eqref{eq:prox_subprob} is known for a wide range of choices of \(\mathcal{R}\) \citep{beck2017first, Rockafellar1998}. PGD has been substantially studied in the deterministic case and is provably convergent to first-order stationary points under weak assumptions \citep{karimi2016linear, teboulle1997convergence}.
In the case \(g_t = \nabla f(x_t)\), \(s_t\) is guaranteed to result in a decrease in \(F\) provided that \(\alpha_t \leq 1 / L\) \citep[Lemma~\(2\)]{palm}, $L$ being the Lipschitz constant of $\nabla f$.\\

Several variants have been successfully adapted to training deep networks and often provide proof of convergence towards critical solutions \citep{davis2019stochastic, pham2020proxsarah, Xu2019, yang2019proxSGD, yun2021adaptive}. They differ in the way they solve~\eqref{eq:prox_subprob},
in whether $\alpha_t$ is fixed or adaptive, in the use of a momentum term, a preconditioner, and other ML techniques that speed up convergence during training.

\subsection{Motivation and proposed approach} 
One notable and common assumption behind the convergence proof of the variants of SG and PGD is the initial learning rate $\alpha_0 \leq 1/L$. In practice, however, especially in deep learning, $L$ is unknown.\\

In the \emph{adaptive quadratic regularization} method, to which we will refer as R2, \(\alpha_t\) is adjusted based on the objective decrease observed at iteration \(t\).
R2 was initially proposed for the case with \(\mathcal{R} = 0\) and the term \emph{regularization} in its name should not be confused with the nonsmooth term \(\mathcal{R}\) in~\eqref{eq:main_pb}. About \(x_t\), a step \(s_t\) is computed that minimizes the linear model \(\varphi(s; x_t) := f(x_t) + \nabla f(x_t)^T s \approx f(x_t + s)\) to which we add the quadratic regularization term \(\tfrac{1}{2} \sigma_t \|s\|^2\), where \(\sigma_t > 0\) is a regularization parameter.
The larger \(\sigma_t\), the shorter we may expect \(s_t\) to be.
Conversely, small values of \(\sigma_t\) may allow us to compute large steps and make fast progress.
By completing the square, note that minimizing \(\varphi(s; x_t) + \tfrac{1}{2} \sigma_t \|s\|^2\) amounts to minimizing \(\tfrac{1}{2} \sigma_t \|s + \sigma_t^{-1} \nabla f(x_t)\|^2\), which corresponds to~\eqref{eq:prox_subprob} with \(\alpha_t := 1 / \sigma_t\) and may be viewed as gradient descent with adaptive step size.\\

\citet{lotfi-bonniot-orban-lodi-2020,lotfi2021adaptive} propose stochastic variants of R2 along with second-order methods for large scale machine learning when $\mathcal{R} = 0$. 
The fact that R2 appears closely relaxed to PGD motivated \citet{aravkin2021proximal} to generalize it to nonsmooth regularized problems with especially weak assumptions on \(\mathcal{R}\).
In the convergence analysis, the value of \(L\) is never explicitly needed.

\paragraph{Organization} The rest of the manuscrip is organized as follows. \Cref{sec:overview} gives a brief overview of ProxSGD and ProxGEN, two proximal methods related to SR2. \Cref{sec:sr2} develops SR2 and justifies the methodology. \Cref{sec:convergence} establishes the convergence guarantees towards a first-order stationary point w.p.1 and an iteration complexity analysis. In \Cref{sec:tests}, we present numerical results and experiments. We conclude with a discussion in \Cref{sec:conclusion}.

\paragraph{Notation}  $\|x\|$ is the Euclidean norm of $x \in \R^n$. $|\mathcal{S}|$ is the number of elements in the set $\mathcal{S}$.
We introduce a stochastic variable \(\xi: \N \to \mathcal{P}(\{1, \ldots, N\}) \setminus \varnothing\), whose domain represents an iteration counter, and which takes values in the set of nonempty samples of the sum in~\eqref{eq:main_pb}.
For a realization \(\xi_t := \xi(t)\) of \(\xi\) at iteration \(t\) we denote
\begin{align*}
    f(x, \xi_t) & := \frac{1}{|\xi_t|} \sum_{j \in \xi_t} f_j(x),
    \quad 
    g_t := \nabla f(x, \xi_t) \phantom{:}= \frac{1}{|\xi_t|} \sum_{j \in \xi_t} \nabla f_j(x)
\end{align*}
the sampled, or stochastic, objective and gradient. We also write $F(x, \xi_t) := f(x, \xi_t) + \mathcal{R}(x)$. We note $\mathbb{E}_\xi$ the expectation over the distribution of $\xi$, while $\mathbb{E}_{\hat{\xi}_t}$ represents the expectation over the distribution of $\xi$ that yields a success knowing $x_t$. The abbreviation w.p.1 means ``with probability one''.

\section{Overview of ProxSGD and ProxGEN}\label{sec:overview}

ProxSGD \citep{yang2019proxSGD} and ProxGEN \citep{yun2021adaptive} are two approaches based on the adaptation of the proximal gradient method, although neither is a descent method. Both consider a variant of~\eqref{eq:prox_subprob} with a momentum term $v_t$ instead of $g_t$, and a preconditioner.\\

ProxSGD assumes that \(\mathcal{R}\) is convex, and computes
\begin{subequations}\label{eq:proxSGD_update}
\begin{align}
    s_t & \in \argmin_s \ v_t^T s + \tfrac{1}{2} s^T D_t s + \mathcal{R}(x_t + s), \\
    x_{t+1} & = x_t + \alpha_t s_t,
\end{align}
\end{subequations}
where \(D_t\) is a positive-definite diagonal matrix. Note that ProxSGD does not exactly fit in the framework~\eqref{eq:prox_subprob}.
\citet{yang2019proxSGD} show convergence to a first-order stationary point w.p.1., but do not provide a complexity bound.\\

Although \citet{yun2021adaptive} do not explicitly mention their assumptions on \(\mathcal{R}\), they mention that ProxGEN does not require it to be convex.
ProxGEN may be seen as a proximal generalization of several SG variants like Adam, Adagrad, etc. that matches~\eqref{eq:prox_subprob} more closely than~\eqref{eq:proxSGD_update}.
It computes
\begin{subequations}\label{eq:proxGEN_update}
\begin{align}
    s_t & \in \argmin_s \ v_t^T s + \tfrac{1}{2} \alpha_t^{-1} s^T D_t s + \mathcal{R}(x_t + s), \\
    x_{t+1} & = x_t + s_t.
\end{align}
\end{subequations}
The authors show convergence to a first-order stationary point, and a worst-case complexity of $\mathcal{O}(\epsilon^{-2})$ in terms of iterations and $\mathcal{O}(\epsilon^{-4})$ overall to achieve $\mathbb{E}_a[\dist(0, \frsb F(x_a))] \leq \epsilon$ when the batch size is fixed, where $x_a$ is an iterate drawn uniformly randomly from $\{x_1, \ldots, x_T\}$, and $T$ is the maximum number of iterations.\\

The method we propose in the next section, SR2, has convergence results similar to ProxGEN but the version we present includes neither a momentum term nor a preconditioner, and it relies on an implicit assumption on the batch size---see \Cref{asp:step} below.

\section{Stochastic quadratic regularization: SR2}\label{sec:sr2}

Recall that \(\mathcal{R}: \R^n \to \R \cup \{\pm \infty\}\) is proper if it never takes the value \(-\infty\) and $\mathcal{R}(x) < \infty$ for at least one $x \in \R^n$, lower semi-continuous at \(\bar{x} \in \R^n\) if \(\liminf_{x \to \bar{x}} \mathcal{R}(x) \geq \mathcal{R}(\bar{x})\), and prox-bounded if there exists $x \in \R^n$ and $\lambda_x >0$ such that $\inf_w \{\frac{1}{2} \lambda_x^{-1} \|x - w\|^2 + \mathcal{R}(w)\} > -\infty$.
The supremum of all such $\lambda_x$ is the threshold of prox-boundedness of $\mathcal{R}$, which we also refer to as \(\lambda_x\). Any function that is bounded below is prox-bounded with \(\lambda_x = +\infty\), but certain unbounded regularizers, such as \(-\|x\|\) or \(-\|x\|^2\), are also prox-bounded. Our assumptions on~\eqref{eq:main_pb} are as follow.

\begin{assumption}\label{asp:basic}
There exists \(L > 0\) such that $f$ is $L$-smooth, i.e., $\|\nabla f(x) - \nabla f(y) \| \leq L \|x-y\|$ for all $x$, $y \in \R^n$.
In addition, \(\mathcal{R}\) is proper and lower semi-continuous at all \(x \in \R^n\), and $s \mapsto \mathcal{R}(x_t + s)$ is prox-bounded for each \(x_t\) encountered during the iterations.
\end{assumption}

Under the previous assumption, the appropriate concept of subdifferential is the following.

\begin{definition} 
    The Fréchet subdifferential \(\frsb \phi(\Bar{x})\) of $\phi: \R^n \to \R \cup \{\pm \infty\}$ at $\Bar{x}$ where \(\phi\) is finite is the set of \(v \in \R^n\) such that
    \[
        \liminf_{\substack{x\rightarrow \Bar{x} \\ x \neq \Bar{x}}} \, \frac{\phi(x) - \phi(\Bar{x}) - v^T (x - \Bar{x})}{\|x - \Bar{x}\|} \geq 0.
    \]
\end{definition}

\begin{assumption}\label{asp:non_empty_frsbdf}
   \(\mathcal{R}\) is such that $\frsb \mathcal{R} \neq \emptyset$, which implies $\frsb F = \nabla f + \frsb \mathcal{R} \neq \emptyset$,
\end{assumption}

Our assumptions on \(\mathcal{R}\) are satisfied for many sparsity-promoting regularizers of interest, including \(\|x\|_0\), \(\|x\|_p\), \(\|x\|_p^p\) for \(0 < p < 1\), and the indicator of \(\{x \mid \|x\|_0 \leq k\}\) for fixed \(k \in \{0, \ldots, n\}\). Note that~\Cref{asp:non_empty_frsbdf} excludes regularizers such as \(-\|x\|\) or \(-\|x\|^2\).\\

As in the deterministic version R2 \citep{aravkin2021proximal}, SR2 uses a linear model of \(f\) defined at each iteration $t$ as $\varphi(s; x_t) = f(x_t, \xi_t) + {g_t}^T s$,
such that $\varphi(0; x_t) = f(x_t, \xi_t)$ and $\nabla_s \varphi(0; x_t) = g_t$.
Let
\begin{equation}
    \label{eq:phi-plus-psi}
    \psi(s; x_t) := \varphi(s ;x_t) + \mathcal{R}(x_t + s).
\end{equation}
Note that the analysis of \citet{aravkin2021proximal} makes provision for using a model of \(\mathcal{R}\) about \(x_t\).
In the interest of clarity, we use the ideal \(\mathcal{R}(x_t + s)\) in the sequel, but our analysis below could just as easily accommodate a model.\\

For a regularization parameter $\sigma_t>0$, we also define
\begin{equation}
    m(s; x_t, \sigma_t) = \psi(s; x_t) + \tfrac{1}{2} \sigma_t \|s\|^2.
    \label{eq:model_sr2}
\end{equation}
SR2 starts the iteration with computing a step $s_t$ that minimizes~\eqref{eq:model_sr2}, which is equivalent to computing a proximal stochastic gradient step with step size ${\sigma_t}^{-1}$:
\begin{equation}
    \label{eq:prox_subprob_sr2}
    s_t \in \argmin_s m(s; x_t, \sigma_t) = \prox_{{\sigma_t}^{-1}\mathcal{R}}({\sigma_t}^{-1} g_t).
\end{equation}
Because the Lipschitz constant of $\nabla \varphi(.;x_t)$ is zero, \(s_t\) is guaranteed to result in a decrease in \(\psi(\cdot; x_t)\) \citep[Lemma~\(2\)]{palm}.
However, the latter does not necessarily correlate with a decrease in \(F\). Therefore, SR2 compares the ratio $\rho_t$ of the decrease in $F(.)$ to that in $\psi(\cdot; x_t)$ between $x_t$ and $x_t + s_t$ to decide on the acceptance of the step.
The value of $\rho_t$, which is indicative of the adequacy of the model \(\psi(\cdot; x_t)\) along \(s_t\), also guides the update of $\sigma_t$.
The procedure is stated in \Cref{alg:sr2}.

\RestyleAlgo{ruled}   
\begin{algorithm}
   \KwIn{ $0 < \eta_1 \leq \eta_2 < 1$, $ 0 < \gamma_3 \leq 1 < \gamma_1 \leq \gamma_2 $, $x_0 \in \R^n$ where $\mathcal{R}$ is finite, $\sigma_0 \geq \sigma_{\min} > 0$}
   \For{$t=1, \ldots $ }{
        Draw \(\xi_t\) and define \(g_t\) \;
        Define $m(s; x_t, \sigma_t)$ as in~\eqref{eq:model_sr2} \;
        Compute $s_t \in \argmin_s m(s;x_t, \sigma_t)$ \;
        \eIf{ $\xi_t$ does not satisfy Assumption 2}{ set $s_t=0$ \;}
        {
            Compute \(\Delta F_t := F(x_t) - F(x_t + s_t)\) \;
            Compute \(\Delta \psi_t := \psi(0;x_t) - \psi(s_t;x_t)\) \;
            Compute \(\rho_t := \Delta F_t / \Delta \psi_t\) \;
            \eIf{$\rho_t \geq \eta_1$}{
            set $x_{t+1} = x_t + s_t$ \Comment*{accept step}
            }{
            set $x_{t+1} = x_t$ \Comment*{reject step}
            }
        }

   Set $\sigma_{t+1} \in
        \left\{
        \begin{array}{lll}
            {[\max(\sigma_{\min}, \gamma_3 \sigma_t), \sigma_t]}          & \text{if } \rho_t \geq \eta_2          & \textcolor{gray}{(\sigma_t \searrow)} \\
            {[\sigma_t, \gamma_1 \sigma_t]}          & \text{if } \eta_1 \leq \rho_t < \eta_2 & \textcolor{gray}{(\sigma_t \approx)} \\ 
            {[\gamma_1 \sigma_t, \gamma_2 \sigma_t]} & \text{if } \rho_t < \eta_1             & \textcolor{gray}{(\sigma_t \nearrow)}
        \end{array}
        \right.$
   }
    \caption{SR2: Stochastic nonsmooth quadratic regularization.}\label{alg:sr2}
\end{algorithm}

The importance of prox-boundedness in \Cref{alg:sr2} resides in the update of $\sigma_t$.
If \(\sigma_t < 1 / \lambda_{x_t}\),~\eqref{eq:model_sr2} is unbounded below, so that \(\Delta m_t = +\infty\).
Because \(\mathcal{R}\) is proper, \(\Delta F_t\) is either finite or \(+\infty\).
Either way, the rules of extended arithmetic in nonsmooth optimization imply \(\rho_t = 0\), and therefore the step is rejected and \(\sigma_t\) is increased.
After a finite number of such increases, \(\sigma_t \geq 1 / \lambda_{x_t}\) and a step that yields finite \(\Delta m_t\) can be assessed.
A key result stated as \Cref{th:sigma} below is that as soon as \(\sigma_t\) is sufficiently large, the step will be accepted.

\Cref{sec:convergence} establishes the convergence properties of SR2, for which we require assumptions that ensure $g_t$ behaves somewhat similarly to \(\nabla f(x_t)\). Comparable conditions appear in \citep{Bottou2018, bollapragada2018adaptive}.

\begin{assumption}
    \label{asp:step}
    There exists $\kappa_m > 0$ such that for all $t$,
        \begin{align*}
            |f(x_t + s_t) - f(x_t)  - {g_t}^Ts_t| &\leq \kappa_m \|s_t\|^2.
        \end{align*}
    In addition
    \(\mathbb{E}_\xi[f(x_t, \xi)] = f(x_t)\), $\mathbb{E_\xi}[g_t] = \nabla f(x_t)$ 
\end{assumption}

\Cref{asp:step} states that the stochastic gradient should behave similarly to a full gradient, which implicitly involves a condition on the batch size.
If the assumption is not respected, the batch-size should be increased.
The process is finite because due to \Cref{asp:basic}, the inequality of \Cref{asp:step} holds with \(\kappa_m = \tfrac{1}{2} L\) when $g_t = \nabla f(x_t)$.
This is similar in spirit to implementing a variance reduction strategy, a standard condition for the convergence of stochastic gradient methods \citep{Bottou2018}.

\section{Convergence analysis}\label{sec:convergence}



Under \Cref{asp:basic}, $x^*$ is first-order stationary for~\eqref{eq:main_pb} if $0 \in \frsb F(x^*) = \nabla f(x^*) + \frsb \mathcal{R}(x^*)$ \citep[Theorem~\(10.1\)]{Rockafellar1998}.\\

The following result mirrors \citep[Theorem~\(6.2\)]{aravkin2021proximal} and shows that SR2 cannot generate a infinite number of failed iterations unless the step is zero. 
We require the following final assumption stating that \(s \mapsto \mathcal{R}(x_t + s)\) are \emph{uniformly} prox-bounded.
The assumption is trivially satisfied for any \(\mathcal{R}\) that is bounded below.

\begin{assumption}\label{asp:unif-prox-bounded}
  There exists \(\lambda > 0\) such that \(\lambda_{x_t} \geq \lambda\) for all \(x_t\) encountered during the iterations.
\end{assumption}

\begin{theorem}\label{th:sigma}
   Let \Cref{asp:basic,asp:step,asp:unif-prox-bounded} hold. If $s_t \neq 0$ and $\sigma_t \geq \sigma_{\textup{succ}} := \max(2\kappa_m / (1-\eta_2), 1 / \lambda)$, then $s_t$ is accepted and $\sigma_t \leq \sigma_t$.
\end{theorem}

\begin{proof}
    As explained above, we assume that \(\sigma_t \geq 1 / \lambda \geq 1 / \lambda_{x_t}\) to ensure that \(\Delta m_t\) is finite.
    By definition of \(s_t\), $m_t(s_t, x_t, \sigma_t) \leq m_t(0, x_t, \sigma_t)$, i.e., 
    \begin{equation}\label{eq:optim_s2}
        {g_t}^Ts_t + \mathcal{R}(x_t + s_t) + \tfrac{1}{2}\sigma_t\|s_t\|^2 \leq \mathcal{R}(x_t).
    \end{equation}
    The definition of $\rho_t$, \Cref{asp:step} and~\eqref{eq:optim_s2} yield
    \begin{equation*}
      |\rho_t - 1| = \abs*{\frac{f(x_t) + {g_t}^Ts_t - f(x_t+s_t)}{\mathcal{R}(x_t) - {g_t}^Ts_t - \mathcal{R}(x_t + s_t)}}
      \leq \frac{2 \kappa_m \|s_t\|^2}{\sigma_t \|s_t\|^2}= \frac{2 \kappa_m}{\sigma_t} \leq \frac{2 \kappa_m}{\sigma_{\textup{succ}}} = 1 - \eta_2.
    \end{equation*}
    Thus, $\rho_t \geq \eta_2$ and $\sigma_t \leq \sigma_t$.
\end{proof}


As a consequence of~\Cref{th:sigma}, there is a constant $\sigma_{\max} := \min \{ \sigma_0, \gamma_2 \sigma_{\textup{succ}}\} >0$ such that for all $t, \sigma_t \leq \sigma_{\max}$.\\

Next, we analyze the scenario where SR2 only generates a finite number of successes, and show that the method converges to a first order stationary point w.p.1 in this case.

\begin{theorem}\label{th:2}
   Let \Cref{asp:basic,asp:step,asp:unif-prox-bounded} hold.
   If \Cref{alg:sr2} only generates a finite number of successes, $x_t=x_{t^*}$ for all sufficiently large $t$ and $x_{t^*}$ is first-order stationary w.p.1.
\end{theorem}

\begin{proof}
  If \Cref{alg:sr2} results in a finite number of successful iterations, there exists $ t_1$ so that for all $t \geq t_1$, iteration $t$ fails. Consequently, $\rho_t < \eta_1$ and $\sigma_{t+1} \geq \gamma_1 \sigma_t $.\\

  Necessarily, there exists a $ t_2 \geq t_1$ such that $\sigma_t \geq \sigma_{\textup{succ}}$ for all \(t \geq t_2\).\\
  
  If there existed \(t \geq t_2\) such that $s_t \neq 0$, Theorem~\ref{th:sigma} would ensure that iteration $t$ is successful, which contradicts our assumption.
  Therefore, $s_t = 0$ and \(0 \in \argmin_s m_{t_2}(s;x_{t_2}, \sigma_{t_2})\).\\
  
  Since $\mathcal{R}$ is prox-bounded, $\frsb \mathcal{R}$ closed and convex \citep[Propositions~$8.6$ and~$8.46$]{Rockafellar1998}, and therefore, \(-g_t \in \frsb \mathcal{R}(x_{t_2})\), for all \(t \geq t_2\).

We now show that $-\nabla f(x_{t_2}) \in \frsb \mathcal{R}(x_{t_2})$.
The empirical mean of the next $m$ stochastic gradients satisfies
\begin{equation*}
    - \Bar{g}_m = -\frac{1}{m} \sum_{i=t_2}^{t_2+m} g_i \in \frsb \mathcal{R}(x_{t_2}),
\end{equation*}
because \(\frsb \mathcal{R}(x_{t_2})\) is convex.\\

According to the law of large numbers and \Cref{asp:step},
\begin{equation*}
    \lim_{m \to \infty} \Bar{g}_m = \mathbb{E}[g] = \nabla f(x_{t_2}) \quad \text{w.p.1}.
\end{equation*}
Because $\frsb \mathcal{R}(x_{t_2})$ is closed, \(- \nabla f(x_{t_2}) \in \frsb \mathcal{R}(x_{t_2})\), i.e.,
\begin{equation*}
    0 \in \nabla f(x_{t_2}) + \frsb \mathcal{R}(x_{t_2}) = \frsb F(x_{t_2}),
\end{equation*}
and $x_{t_2}$ is a first order stationary point w.p.1.
\end{proof}

We now focus on the case where SR2 generates infinitely many successes. By analogy with the deterministic and smooth case where \(s_t = -\sigma_t^{-1} \nabla f(x_t)\), our criticality measure is $\mathbb{E}_{\hat{\xi}_t} [\|s^\xi\|^2] \leq \epsilon^2$, where $\mathbb{E}_{\hat{\xi}_t}$ denotes the expectation taken over the distribution of the $\xi$ that yields a success knowing $x_t$. The first iteration that satisfies the latter condition is noted $t(\epsilon)$. 

We start by studying the complexity of reaching this termination criteria. To that effect, let us define
\begin{align}
    \mathcal{S} &= \{t\in \N  \mid  \rho_t \geq \eta_1 \}, \\
    \mathcal{S(\epsilon)} &= \{t\in \mathcal{S} \mid t \leq t(\epsilon) \}, \\
    \mathcal{U(\epsilon)} &= \{t\in \N \mid  t < t(\epsilon) \text{ and } \rho_t < \eta_1\}.
\end{align}

\begin{lemma}\label{lemma:s_epsilon}
    Let \Cref{asp:basic,asp:step,asp:unif-prox-bounded} hold.
    If \Cref{alg:sr2} generates an infinite number of successes and if there exists \(F_{\textup{low}} \in \R\) such that $F(x_t) \geq F_{\textup{low}}$ for all $t \geq 0$, then for any $\epsilon \in (0,1)$,
    $|\mathcal{S(\epsilon)}| = \mathcal{O}(\epsilon^{-2})$.
\end{lemma}

\begin{proof}

When $t \in \mathcal{S(\epsilon)}$, $\rho_t \geq \eta_1$.

  Using~\eqref{eq:optim_s2}, the facts that $t < t(\epsilon)$ and \(\sigma_{\min} \leq \sigma_t \leq \sigma_{\max}\), we have
 
\begin{equation*}
    F(x_t) - F(x_t + s_t) \geq \eta_1 \Delta \psi_t
    \geq \tfrac{1}{2} \eta_1 \sigma_{\min} ||s_t||^2.
\end{equation*}

This inequality holds for every $s^{\xi}$ derived from $\xi$ that yields a success at iteration $t$. We can therefore introduce the expectation over the distribution of the $\xi$ that yield a success given $x_t$, denoted $\mathbb{E}_{\hat{\xi}_t}$. Therefore $\mathbb{E}_{\hat{\xi}_t}[F(x_t + s^{\xi})]$ is a relevant quantity, and the previous inequality becomes
\begin{equation}
    F(x_t) - \mathbb{E}_{\hat{\xi}_t}[F(x_t + s^{\xi})] \geq \eta_1 \Delta \psi_t
    \geq \tfrac{1}{2} \eta_1 \sigma_{\min} \mathbb{E}_{\hat{\xi}_t}[||s_t||^2].
    \label{eq:interm_eq}
\end{equation}

Because $t<t(\epsilon)$, $\mathbb{E}_{\hat{\xi}_t}[||s_t||^2] \geq \epsilon^2$.
Thus, since $t\in \mathcal{S}$,~\eqref{eq:interm_eq} becomes
\begin{equation}\label{eq:expected_decrease_t}
    F(x_t) - \mathbb{E}_{\hat{\xi}_t}[F(x_t + s^{\xi})] =
    F(x_t) - \mathbb{E}_{\hat{\xi}_t}[F(x_{t+1})] \geq
    \tfrac{1}{2} \eta_1 \sigma_{\min} \epsilon^2.
\end{equation}

By analogy with \citet{Bottou2018}, we introduce the total expectation $\mathbb{E}[.]$ with respect to the joint distribution of all previous realization of $\xi$ that yield a success, thus $\mathbb{E}[F(x_t)] := \mathbb{E}_{\hat{\xi}_1} \mathbb{E}_{\hat{\xi}_2} \ldots \mathbb{E}_{\hat{\xi}_{t-1}}[F(x_t)]$.
Taking the total expectation in~\eqref{eq:expected_decrease_t} yields
\begin{equation}\label{eq:expected_decrease_total}
    \mathbb{E}[F(x_t)] - \mathbb{E}[F(x_{t+1})] \geq  \tfrac{1}{2} \eta_1 \sigma_{\min} \epsilon^2.
\end{equation}

Because $x_{t+1} = x_t$ if \(\xi_t\) yields \(t \in \mathcal{U}\), while $x_{t+1} = x_t + s_t$ if \(\xi_t\) yields \(t \in \mathcal{S}\),
    
\begin{align*}
    F(x_1) - F(x_{\textup{low}})&\geq  \mathbb{E}[F(x_1)] - \mathbb{E}[F(x_{t(\epsilon)})] \\
    &= \sum_{t=1}^{t(\epsilon)} \Big[ \mathbb{E}[F(x_t)] - \mathbb{E}[F(x_{t+1})] \Big]\\
        &\geq \sum_{t \in \mathcal{S}(\epsilon)} \Big[ \mathbb{E}[F(x_t)] - \mathbb{E}[F(x_{t+1})] \Big] \\
        &\geq \tfrac{1}{2} \eta_1 \sigma_{\min} \epsilon^2 |\mathcal{S}(\epsilon)|.
  \end{align*}

Therefore, $|\mathcal{S}(\epsilon)| = \mathcal{O}(\epsilon^{-2})$.
\end{proof}

\begin{lemma}\label{lemma:u_epsilon}
    Under the assumptions of \Cref{lemma:s_epsilon}, $|\mathcal{U(\epsilon)}| = \mathcal{O}(\epsilon^{-2})$.
\end{lemma}

\begin{proof}
Let $t \in \mathcal{U(\epsilon)}$, so that $t < t(\epsilon)$. \Cref{alg:sr2} increases $\sigma_t$ by a factor of at least $\gamma_1 > 1$ if the step is rejected, and decreases $\sigma_t$ by a factor of at most $\gamma_3 \in (0, 1]$ if it is accepted.
 
Thus, at iteration $t(\epsilon) - 1$, we have successively
\begin{alignat*}{2}
  & & \sigma_{\max} & \geq \sigma_{t(\epsilon) - 1} \geq \sigma_0 \gamma_1^{|\mathcal{U(\epsilon)}|} \gamma_3^{|\mathcal{S(\epsilon)}|}\\
  & \Rightarrow & \dfrac{\sigma_{\max}}{\sigma_0} & \geq \gamma_1^{|\mathcal{U(\epsilon)}|} \gamma_3^{|\mathcal{S(\epsilon)}|}\\
  & \Leftrightarrow & \log(\dfrac{\sigma_{\max}}{\sigma_0}) & \geq |\mathcal{U(\epsilon)}| \log(\gamma_1) + |\mathcal{S(\epsilon)}| \log(\gamma_3)\\
  & \Leftrightarrow \quad & |\mathcal{U(\epsilon)}| \log(\gamma_1) & \leq \log(\dfrac{\sigma_{\max}}{\sigma_0}) - |\mathcal{S(\epsilon)}| \log(\gamma_3).
 \end{alignat*}
 Because $\log(\gamma_3) < 0$ and $|\mathcal{S(\epsilon)}| = \mathcal{O}(\epsilon^{-2})$, we obtain $|\mathcal{U(\epsilon)}| = \mathcal{O}(\epsilon^{-2}).$
\end{proof}

From $t(\epsilon) = |\mathcal{S(\epsilon)}| + |\mathcal{U(\epsilon)}|$, we deduce $t(\epsilon) = \mathcal{O}(\epsilon^{-2})$,
and obtain the two following results.

\begin{theorem}\label{th:3}
   Under the assumptions of~\Cref{lemma:s_epsilon}, either $F$ is unbounded from below or $\liminf_{t \rightarrow \infty} \mathbb{E}_{\hat{\xi}_t}[||s^\xi||^2] = 0$
\end{theorem}

\begin{theorem}\label{th:4}
    Let \(0 < \epsilon < 1\). Then,  
    $\mathbb{E}_{\hat{\xi}_{t(\epsilon)}}[\dist(0, \frsb F(x_{t(\epsilon)}+s^{\xi}))^2] \leq C \epsilon^2 + 3 \mathbb{E}_{\hat{\xi}_{t(\epsilon)}}[||\nabla f(x_{t(\epsilon)}) - g^\xi||^2]$, where $C=3(L^2+\sigma_{\max}^2)$.
\end{theorem}

\begin{proof}
From the definition of $s_t$, we have
\begin{align*}
    0 &\in g_t + \sigma_t s_t + \frsb \mathcal{R}(x_t+s_t)\\
    \Leftrightarrow -(g_t+\sigma_t s_t) &\in \frsb \mathcal{R}(x_t+s_t)\\
    \Leftrightarrow \nabla f(x_t+s_t) -(g_t+\sigma_t s_t) &\in \frsb F(x_t+s_t).
\end{align*}

Thus
\begin{align*}
    \dist(0, \frsb F(x_t+s_t))^2 &\leq \| \nabla f(x_t+s_t) - g_t - \sigma_t s_t \|^2\\
            &\leq 3\|\nabla f(x_t+s_t) - \nabla f(x_t) \|^2 + 3\|\sigma_t s_t\|^2 + 3\| \nabla f(x_t)- g_t \|^2 \\
            & \leq 3 L^2\|s_t\|^2 + 3\|\sigma_t s_t\|^2 + 3\| \nabla f(x_t)- g_t \|^2 \\
            & \leq 3 (L^2 + \sigma_{\max}^2 ) \|s_t\|^2 + 3\| \nabla f(x_t)- g_t \|^2,
\end{align*}
which is true for every step $s^{\xi}$ computed with a realization of $\xi$, i.e.,
\begin{equation*}
    \dist(0, \frsb F(x_t+s^{\xi}))^2 \leq 3 (L^2 + \sigma_{\max}^2 ) \|s^{\xi}\|^2 + 3\| \nabla f(x_t)- g^{\xi} \|^2.
\end{equation*}
Therefore 
\begin{equation*}
    \mathbb{E}_{\xi_{t(\epsilon)}}[\dist(0, \frsb F(x_t+s^{\xi}))^2] \leq 3 (L^2 + \sigma_{\max}^2 ) \mathbb{E}_{\xi_{t(\epsilon)}}[\|s^{\xi}\|^2] + 3\mathbb{E}_{\xi_{t(\epsilon)}} [ \| \nabla f(x_t)- g^{\xi} \|^2].
\end{equation*}

For $t=t(\epsilon)$, we have $\mathbb{E}_{\xi_{t(\epsilon)}} [\|s^\xi\|^2]\leq \epsilon^2$. Thus
\begin{equation*}
    \mathbb{E}_{\xi_{t(\epsilon)}}[\dist(0, \frsb F(x_t+s^{\xi}))^2] \leq 3 (L^2 + \sigma_{\max}^2 ) \epsilon^2 + 3\mathbb{E}_{\xi_{t(\epsilon)}} [\| \nabla f(x_t)- g^{\xi} \|^2]. \qedhere
\end{equation*}

\end{proof}

Finally, we are set to analyze the properties of $x_{t(\epsilon)}$ in terms of stationarity. 
\citet{shamir2020can} and \citet{Zhang2020OnCO} discuss the impossibility of finding $\epsilon$-stationary points for nonsmooth and nonconvex functions with first-order methods in finite time.
Instead, \citet{Zhang2020OnCO} introduce a relaxation of the concept of $\epsilon$-stationarity, namely, $(\delta, \epsilon)$-stationarity which is reported in \Cref{def:d_e_stationarity}.

\begin{definition}\label{def:d_e_stationarity}
    A point $x$ is called $(\delta, \epsilon)$-stationary if 
    \(
        d(0, \partial F(x+\delta B)) \leq \epsilon,
    \)
    where $\partial F(x+\delta B) := conv \big(\cup_{y \in x + \delta B} \partial F(y) \big)$, and $\partial F$ is the generalized gradient of $F$~\citep{clarke1990optimization}.
\end{definition}

We propose a variant of \Cref{def:d_e_stationarity} that is better adapted to our method. Note that other adaptations of the $(\delta, \epsilon)$-stationarity notion are discussed in \citet{shamir2020can}. 

\begin{definition}\label{def:our_d_e_stationarity}
    A point $x$ is called $\widehat{(\delta, \epsilon)}$-stationary if 
    \(
        \mathbb{E}_{\hat{\xi}} \Big [ d(0, \frsb F(x+s^\xi))^2\Big ] \leq \epsilon^2,  \text{ with } \delta = \max_{\hat{\xi}} ||s^\xi||.
    \)
\end{definition}

\Cref{def:our_d_e_stationarity} appears in the result of \Cref{th:4}, if a variance reduction strategy is additionally implemented to ensure the second right term becomes lower than $\epsilon^2$. This remark is expressed in \Cref{cor:1}.

\begin{corollary}\label{cor:1}
    Let \(0 < \epsilon < 1\). If a variance reduction strategy ensures $\mathbb{E}_{\hat{\xi}_{t(\epsilon)}}[||\nabla f(x_{t(\epsilon)}) - g^\xi||^2] \leq \epsilon^2$, then $x_{t(\epsilon)}$ is a $\widehat{(\delta, (C+3)\epsilon)}$ stationary point, with $C = 3(L^2+\sigma_{\max}^2)$ and $\delta = \max_{\hat{\xi}} ||s^\xi||$.
\end{corollary}

\section{Experiments}\label{sec:tests}

We compare SR2 against ProxSGD and ProxGEN to train three DNNs on the CIFAR-10 and CIFAR-100 datasets. The networks considered are DenseNet-121, ResNet-34 and DenseNet-201, with $7.98$M, $21.79$M and $20$M parameters respectively. Each set of tests uses $\mathcal{R} = \lambda \|\cdot\|_1, \lambda \in \{10^{-4}, 10^{-3}, 10^{-2}\}$, while $\mathcal{R} = \lambda \|\cdot\|_0$ is not tested with ProxSGD as it is not designed for nonconvex regularization.\\

We use the proximal SGD variant of ProxGEN. The implementation of ProxGEN was provided to us by its authors, and we also use their implementation of ProxSGD. Both methods use the hyperparameters mentioned in their respective papers and implementations. The implementation of SR2 is available at \url{https://github.com/DouniaLakhmiri/SR2} and its configurationis reported in \Cref{tab:hps}. In our implementation, we compute $\rho_t$ based on the sampled value of $F(., \xi_t)$ instead of the full objective.\\ 

For the sake of a fair comparison, we have disabled the momentum directions and preconditioners from ProxSGD and ProxGEN as well as the scheduled updates of the learning rate at epochs $150$ and $250$. These common accelerating strategies are not yet incorporated to SR2 and would give ProxSGD and ProxGEN an unfair advantage as shown in~\Cref{fig:acc_jump}. \\

Each test trains for $300$ epochs after which we proceed to pruning each solution based on the criterion $|w_i| \leq \alpha $ with $\alpha = 10 ^{-k}$, \(k = 1, \ldots, 8\), where $w_i$ is the $i$-th weight in the network. We then compare the sparsity level and retained accuracy of the sparse networks without re-training.

\begin{minipage}[ht]{\textwidth}
  \begin{minipage}[b]{0.49\textwidth}
        \centering
        \includegraphics[scale=0.5]{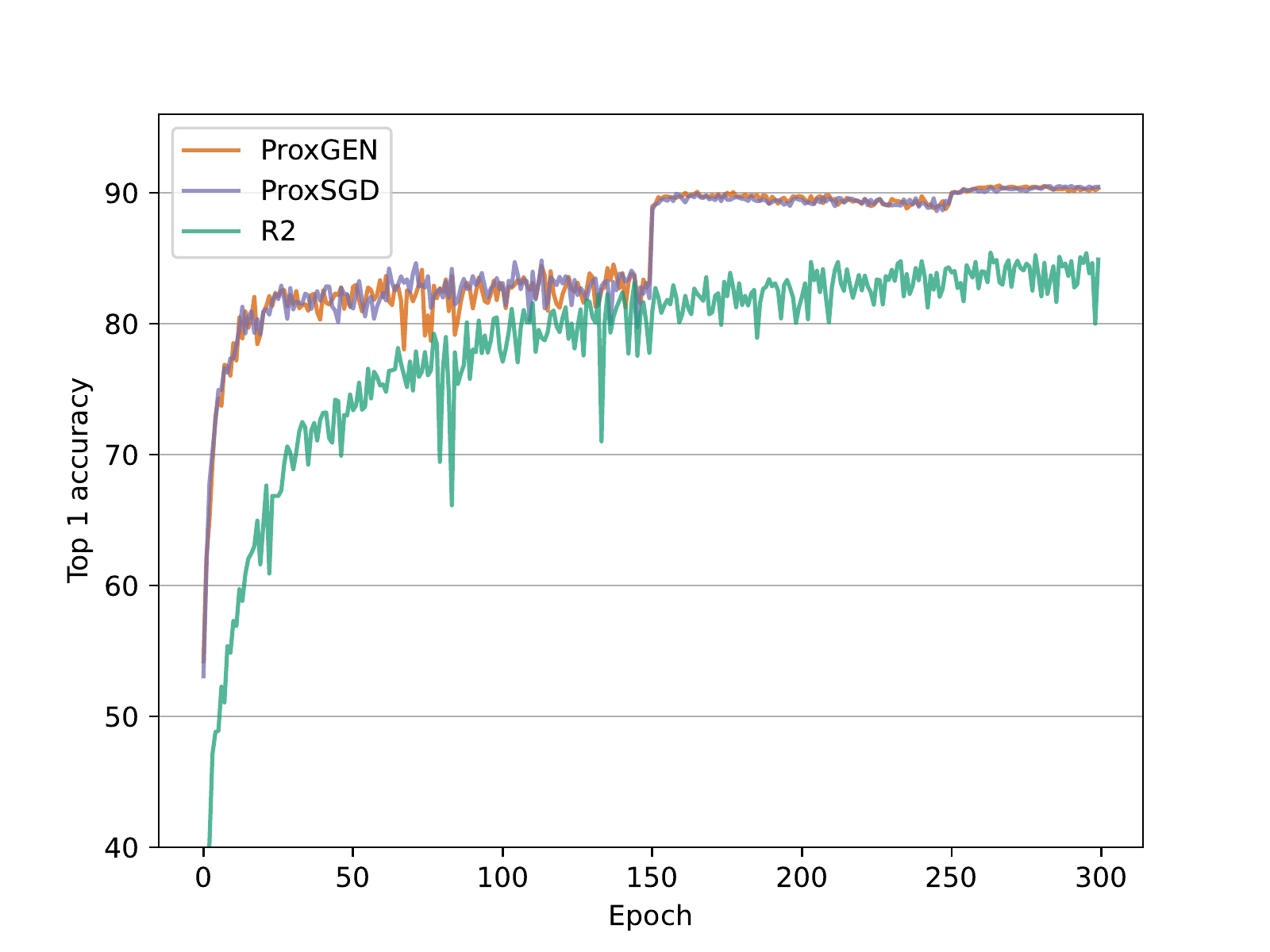}
        \captionof{figure}{Accuracy of ProxGEN and ProxSGD with momentum, preconditioner, and learning rate updates at epochs $150$ and $250$.}
    \label{fig:acc_jump}
  \end{minipage}
  \hfill
  \begin{minipage}[b]{0.49\textwidth}
    \centering
    \begin{tabular}{lrrrr}
        \toprule
        $\eta_1$ & $\eta_2$ & $\gamma_1$ & $\gamma_2$ & $\gamma_3$ \\
        \midrule
        $7.5 \cdot 10^{-4}$ & $0.99$ & $5.56$ & $2.95$ & $0.8$\\
        \bottomrule
    \end{tabular}
    \captionof{table}{SR2 hyperparameters.}
    \label{tab:hps}
  \end{minipage}
\end{minipage}


\subsection{Results on CIFAR-10}

\Cref{tab:res_c10} reports the results with $\mathcal{R} = \lambda \|\cdot\|_1$. For both networks, we observe that SR2 combined with $\lambda = 10^{-4}$ achieves the highest accuracy overall, while ProxSGD gets the highest accuracies with $\lambda = 10^{-3}, 10^{-2}$. \Cref{tab:res_c10} also reports information on the magnitudes of the weights in each solution. Interestingly, SR2 has a consistent tendency to set a large portion of the network's weights to exactly 0 while ProxSGD does the opposite and ProxGEN falls in between in this regard. This observation highlights the clear difference between ProxSGD and ProxGEN in the solutions each method finds. In addition, SR2 identifies a larger proportion of small weights than ProxSGD and ProxGEN.\\

\Cref{fig:d121} (top) reports accuracy and sparsity results on pruned DenseNet-121  with $\mathcal{R} = \lambda\|\cdot\|_1$. The top plot shows that most configurations retain full accuracy until $\alpha = 10^{-3}$ or $10^{-2}$, except for the one trained with ProxSGD, which shows a small drop at $\alpha = 10^{-3}$. The accuracy of all networks drops to $10\%$ for $\alpha = 10^{-1}$.
The plot at the top right shows the sparsity ratio with each pruning criteria. Overall, the combination of SR2 with $\lambda=10^{-4}$ and $\alpha = 10^{-2}$ has the highest accuracy with a high sparsity level of $97.5\%$, followed by ProxGEN with $\lambda=10^{-3}$ and $\alpha = 10^{-3}$ and a sparsity of $94.4\%$.\\

\begin{table}[!htb]
    \caption{Results of training DenseNet-121 and ResNet-34 on CIFAR-10}
    \begin{subtable}[t]{.5\linewidth}
      \caption{with $\mathcal{R} = \lambda\|\cdot\|_1$.}
     \label{tab:res_c10}
    \resizebox{\columnwidth}{!}{%
    \begin{tabular}{lcccrrr} \toprule
    \textbf{Net.} & \textbf{$\lambda$} & \textbf{Optim.} & \textbf{Acc.} & $\% |w| = 0$ & $\% |w| \leq 10^{-3}$  \\
        \midrule
         & & ProxSGD & $72.20 \%$ & $0.00\%$ & $20.22\%$\\
         & $10^{-4}$ & ProxGEN & $ 72.26\%$  & $6.50\%$  & $22.71\%$\\
                & & SR2  & $\mathbf{84.69 \%}$& $\mathbf{79.47}\%$ & $\mathbf{92.15}\%$\\
        \cline{2-6}
         & & ProxSGD & $\mathbf{77.43 \%}$ & $0.00\%$ & $82.18\%$\\
        D-121 & $10^{-3}$ & ProxGEN & $ 76.81\%$  & $44.21\%$  & $94.40\%$\\
                & & SR2  & $68.26 \%$& $\mathbf{95.43\%}$ & $\mathbf{98.17}\%$\\
        \cline{2-6}        
         & & ProxSGD & $\mathbf{78.36 \%}$ & $0.00\%$ & $94.16\%$\\
         & $10^{-2}$ & ProxGEN & $ 59.69\%$  & $\mathbf{98.03\%}$  & $\mathbf{99.13\%}$\\
                & & SR2  & $76.49 \%$& $78.11\%$ & $98.59\%$\\
        \bottomrule
         & & ProxSGD & $85.12 \%$ & $0.00\%$ & $64.70\%$\\
         & $10^{-4}$ & ProxGEN & $ 85.98\%$  & $1.02\%$  & $72.94\%$\\
                & & SR2  & $\mathbf{93.94 \%}$& $\mathbf{56.18\%}$ & $\mathbf{98.12\%}$\\
        \cline{2-6}
         & & ProxSGD & $\mathbf{89.67\%}$ & $0.00\%$ & $94.85\%$\\
        R-34 & $10^{-3}$ & ProxGEN & $83.27\%$  & $\mathbf{73.34\%}$ &$99.25\%$ \\
               &    & SR2  & $88.46\%$ & $67.97\%$ & $\mathbf{99.50\%}$\\
        \cline{2-6}
         & & ProxSGD & $\mathbf{88.12 \%}$ & $0.00\%$ & $98.28\%$\\
         & $10^{-2}$ & ProxGEN & $ 35.56\%$  & $\mathbf{99.34\%}$  & $\mathbf{99.92\%}$\\
                & & SR2  & $29.33 \%$& $62.17\%$ & $91.84\%$ \\
        \bottomrule
    \end{tabular}
    }
    \end{subtable}%
    \hfill
    \begin{subtable}[t]{.5\linewidth}
      \caption{with $\mathcal{R} = \lambda \|\cdot\|_0$.}
     \label{tab:res_c10_l0}
    \resizebox{\columnwidth}{!}{%
     \begin{tabular}{lcccrrr} \toprule
    \textbf{Net.} & \textbf{$\lambda$} & \textbf{Optim.} & \textbf{Acc.} & $\% |w| = 0$ & $\% |w| \leq 10^{-3}$  \\
        \midrule
            & $10^{-4}$ & ProxGEN & $71.09 \%$  &$2.39\%$ &  $3.27\%$\\
            &                   &  SR2    & $\mathbf{80.29 \%}$ & $\mathbf{14.67\%}$ & $\mathbf{14.74 \%}$\\ 
        \cline{2-6}
         & $10^{-3}$ & ProxGEN & $70.44 \%$  &$4.13\%$ &  $4.13\%$\\
        D-121   &   &  SR2  & $  \mathbf{79.11 \%}$ & $\mathbf{23.82\%}$ & $ \mathbf{25.85\%}$\\ 
        \cline{2-6}
         & $10^{-2}$ & ProxGEN & $71.03 \%$  &$9.63\%$ &  $10.05\%$\\
           &   &  SR2  & $\mathbf{79.79\%}$ & $\mathbf{94.99\%}$ & $ \mathbf{95.06\%}$\\ 
        \bottomrule
        & $10^{-4}$ & ProxGEN & $86.87 \%$  &$5.48\%$ &  $7.13\%$\\
        &                   &  SR2    & $\mathbf{90.59 \%}$ & $\mathbf{93.43\%}$ & $\mathbf{93.43 \%}$\\ 
        \cline{2-6}
         & $10^{-3}$ & ProxGEN & $85.81\%$ & $11.04\%$ & $11.07\%$  \\
        R-34    &    &    SR2  & $\mathbf{92.20\%}$ &$\mathbf{94.41\%}$ & $\mathbf{94.42\%}$\\
        \cline{2-6}
        & $10^{-2}$ & ProxGEN & $86.40 \%$  &$28.85\%$ &  $28.86\%$\\
        &                   &  SR2    & $\mathbf{87.82 \%}$ & $\mathbf{99.04\%}$ & $\mathbf{99.07 \%}$\\ 
        \bottomrule
    \end{tabular}
    }
    \end{subtable} 
\end{table}

\begin{figure}[h]
    \centering
    \includegraphics[scale=0.3]{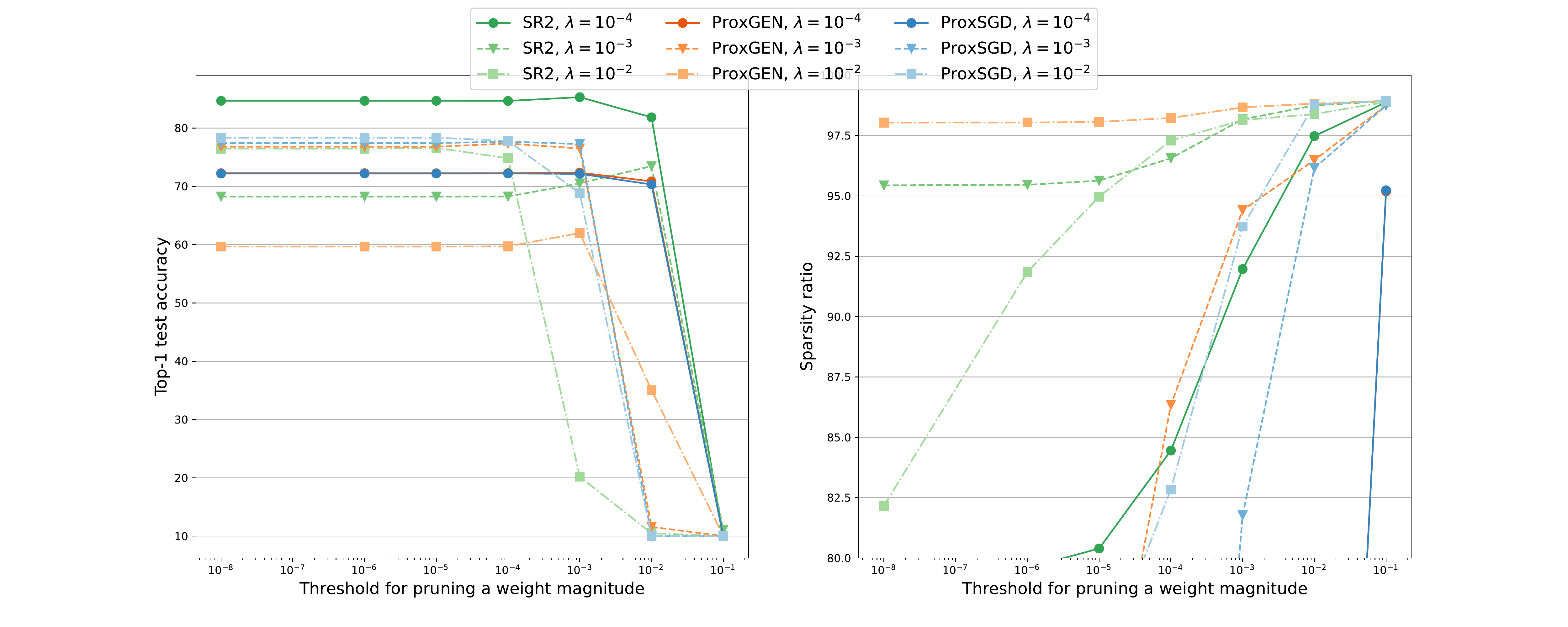}
    \\
    \includegraphics[scale=0.3]{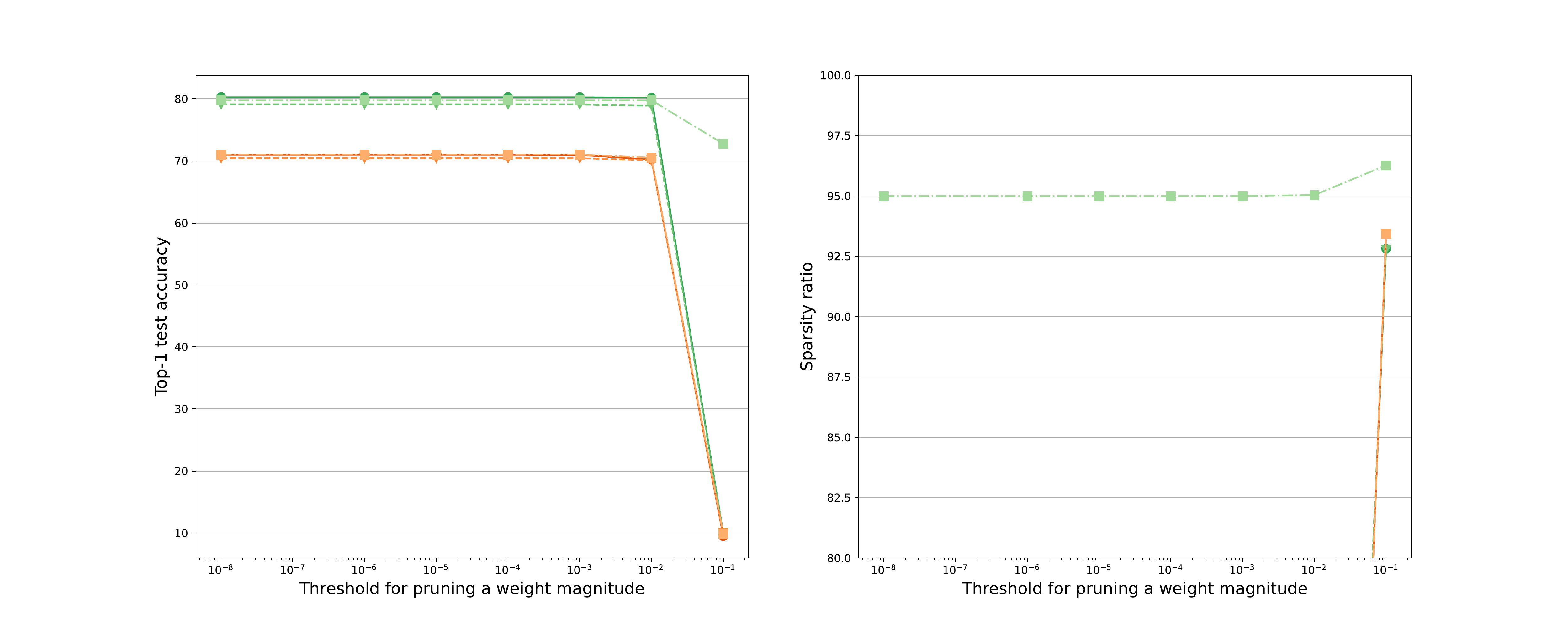}
    
      \caption{\label{fig:d121}%
      Accuracy and sparsity ratio of pruned DenseNet-121 on CIFAR-10 with $\mathcal{R} = \lambda\|\cdot\|_1$ (top) and $\mathcal{R} = \lambda\|\cdot\|_0$ (bottom).}
\end{figure}

\begin{figure}[h]
    \includegraphics[scale=0.3]{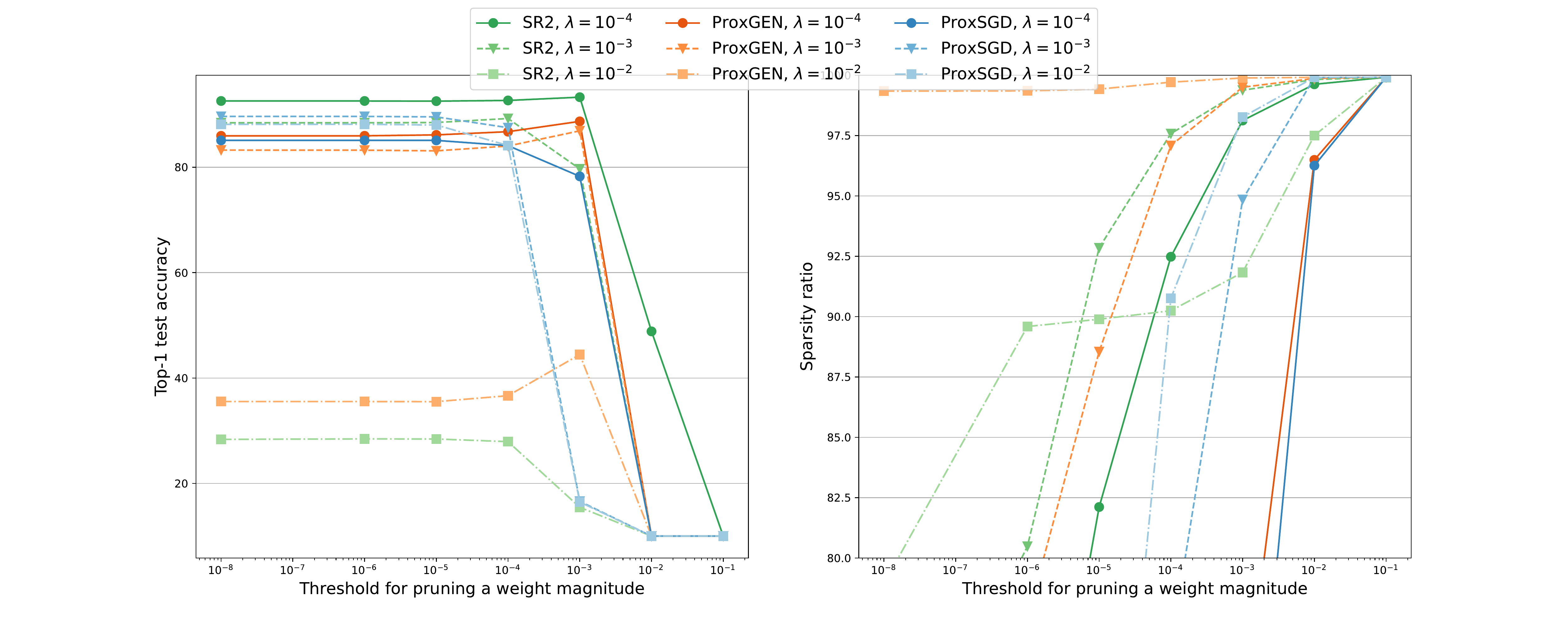}
    \includegraphics[scale=0.3]{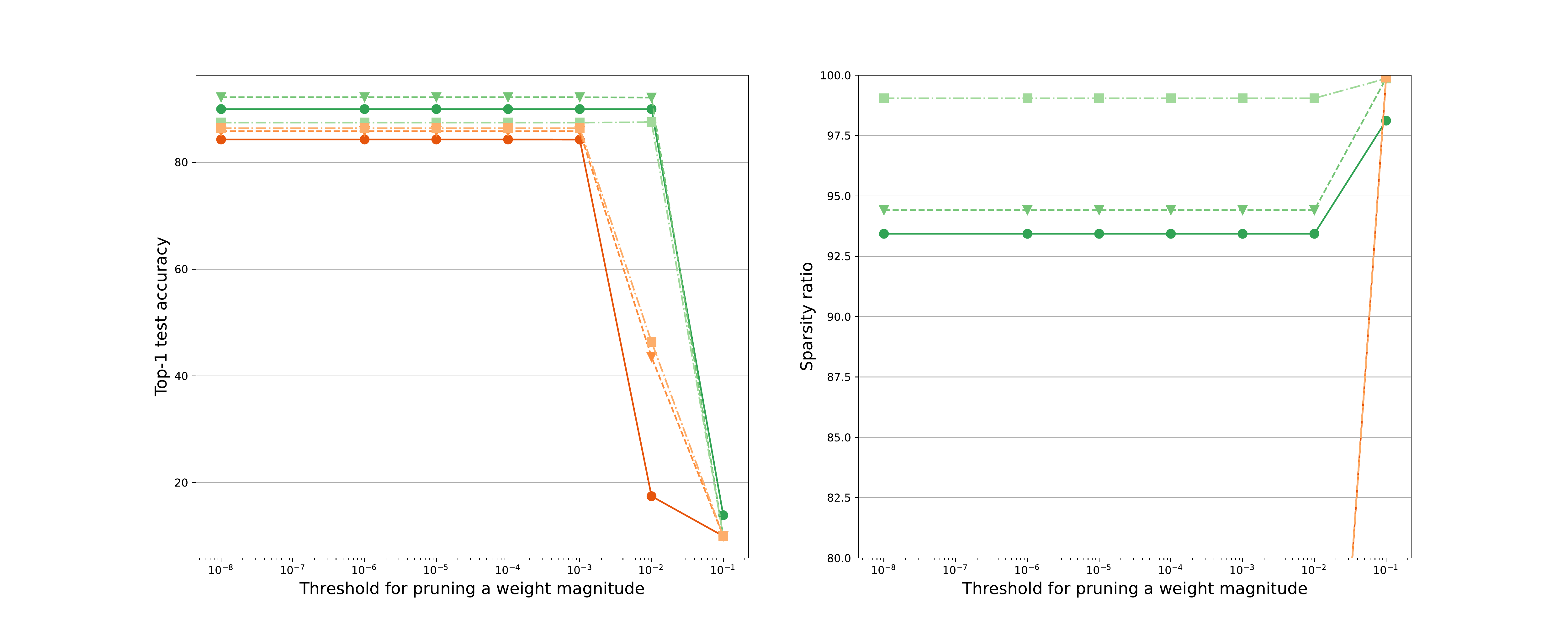}

      \caption{\label{fig:r34}%
      Accuracy and sparsity ratio of ResNet-34 trained on CIFAR-10 with $\mathcal{R} = \lambda\|\cdot\|_1$ (top) and $\mathcal{R} = \lambda\|\cdot\|_0$ (bottom).}
\end{figure}

\Cref{fig:r34} (top) shows that ResNet-34 retains full accuracy with $\alpha = 10^{-3}$ in most cases. The best combination is obtained with SR2, $\lambda = 10^{-4}$ and $\alpha = 10^{-3}$ that results in an accuracy of $93.32\%$ and a sparsity ratio of $98.12\%$, followed by ProxGEN with $\lambda = 10^{-3}$ and $\alpha = 10^{-3}$ with an accuracy of $86.93\%$ and a sparsity of $99.50\%$.\\

\Cref{tab:res_c10_l0} and \Cref{fig:d121} (bottom) report results on the same networks with $\mathcal{R} = \|\cdot\|_0$ and compares SR2 against ProxGEN only, since ProxSGD does not handle nonconvex regularizations. The results show a clear advantage of SR2 both in terms of final accuracy and sparsity ratios. Compared to $\mathcal{R} = \|\cdot\|_1$, using $\mathcal{R} = \|\cdot\|_0$ allows SR2 to reach higher accuracies overall at the expense of higher weight magnitudes. The results seem to suggest that the value of $\lambda$ needs a special adjustment for each regularizer. \Cref{fig:r34} (bottom) summarizes the retained accuracy after pruning and the equivalent sparsities for ResNet-34. It is clear that SR2 generates the better solutions with the highest sparsity levels while retaining most of the full accuracies. A similar figure for DenseNet-121 is reported in the appendix.

\subsection{Results on CIFAR-100}

In this section, SR2 is compared against ProxSGD and ProxGEN on a more challenging dataset. We train DenseNet-201 on CIFAR-100 with $\mathcal{R} = \|\cdot\|_1$ and $\|\cdot\|_0$ and compare each solution's resulting accuracy and sparsity. Once again, our goal is to extract sparse substructures, and we do not focus our resources on tuning each method to reach high test accuracies.\\ 

\Cref{tab:res_c100} summarizes the relevant scores of each solution with  $\mathcal{R} = \lambda\|\cdot\|_1$, and \Cref{fig:d201} (top) illustrates the retained accuracy and equivalent sparsity ratio after each pruning. SR2 with $\lambda = 10^{-4}, \alpha=10^{-2}$ obtains the highest accuracy of $58.50\%$ after removing $97.57\%$ of the weights from the original network. Other solvers that obtain a higher sparsity after pruning do so at the expense of the final accuracy of the network.\\

Similarly \Cref{tab:res_c100_l0} and \Cref{fig:d201} (bottom) show that SR2 obtains the best accuracy when the network is trained with $\mathcal{R} = \lambda=10^{-3}\|\cdot\|_0$ allows to consistently reach higher sparsity ratios while maintaining at least the accuracy of the full network . The best solution is found with $\lambda = 10^{-3}$ and $\alpha=10^{-2}$ as shown in  \Cref{fig:d201} (bottom).

\begin{table}[!htb]
    \caption{Results of DenseNet-201 on CIFAR-100}
    \begin{subtable}[t]{.5\linewidth}
      \centering
        \caption{with $\mathcal{R} = \lambda\|\cdot\|_1$.}
        \label{tab:res_c100}
        \resizebox{\columnwidth}{!}{%
        \begin{tabular}{lccrrr} \toprule
        \textbf{$\lambda$} & \textbf{Optim.} & \textbf{Acc.} & $\% |w| = 0$ & $\% |w| \leq 10^{-3}$  \\
        \midrule
                    & ProxSGD & $42.38 \%$ & $0.00\%$ & $25.23\%$\\
         $10^{-4}$  & ProxGEN & $ 41.11\%$  & $2.95\%$  & $4.01\%$\\
                    & SR2  & $\mathbf{57.63 \%}$& $\mathbf{59.74}\%$ & $\mathbf{92.65}\%$\\
        \midrule
                    & ProxSGD & $42.38 \%$ & $0.00\%$ & $72.47\%$\\
       $10^{-3}$    & ProxGEN & $\mathbf{46.70\%}$  & $48.94\%$  & $96.99\%$\\
                    & SR2  & $33.04 \%$& $\mathbf{97.21\%}$ & $\mathbf{98.92}\%$\\
        \midrule     
                    & ProxSGD & $\mathbf{42.86 \%}$ & $0.00\%$ & $25.21\%$\\
         $10^{-2}$  & ProxGEN & $ 6.96\%$  & $98.83\%$  & $99.48\%$\\
                    & SR2  & $7.31 \%$& $98.30\%$ & $99.60\%$\\
        \bottomrule
    \end{tabular}
    }
    \end{subtable}%
    \hfill
    \begin{subtable}[t]{.5\linewidth}
      \centering
        \caption{with $\mathcal{R} = \lambda\|\cdot\|_0$.}
        \label{tab:res_c100_l0}

        \resizebox{\columnwidth}{!}{%
        \begin{tabular}{lccrrr} \toprule
            \textbf{$\lambda$} & \textbf{Optim.} & \textbf{Acc.} & $\% |w| = 0$ & $\% |w| \leq 10^{-3}$  \\
            \midrule
            $10^{-4}$  & ProxGEN & $ 40.77\%$  & $2.97\%$  & $4.02\%$\\
                    & SR2  & $\mathbf{48.91} \%$& $ \mathbf{22.69} \%$ & $ \mathbf{22.78}\%$\\
            \midrule
            $10^{-3}$    & ProxGEN & $40.44 \% $  & $5.11\%$  & $5.60\%$\\
                    & SR2  & $\mathbf{49.28 \%}$& $\mathbf{39.82\%}$ & $\mathbf{39.92}\%$\\
            \midrule     
            $10^{-2}$  & ProxGEN & $ 39.91\%$  & $11.84\%$  & $12.33\%$\\
                    & SR2  & $1.00 \%$& $98.91\%$ & $99.50\%$\\
            \bottomrule
        \end{tabular}
        }
    \end{subtable} 
\end{table}

Overall, the results on CIFAR-100 are more contrasted than on CIFAR-10 with examples of ProxGEN and SR2 converging in some settings towards solutions with low accuracy. This suggests the need for a better tuning of the methods.

\begin{figure}[ht]
    \centering
        \includegraphics[scale=0.3]{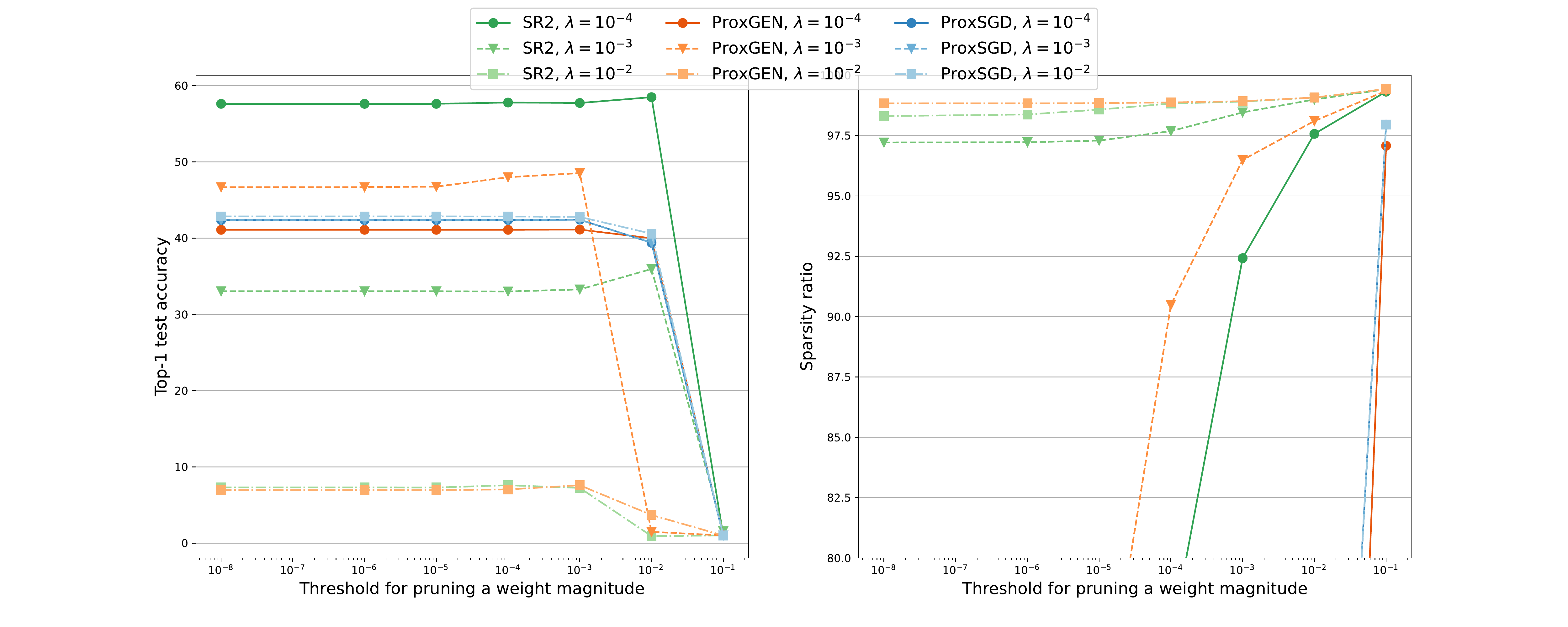}\\
        \includegraphics[scale=0.3]{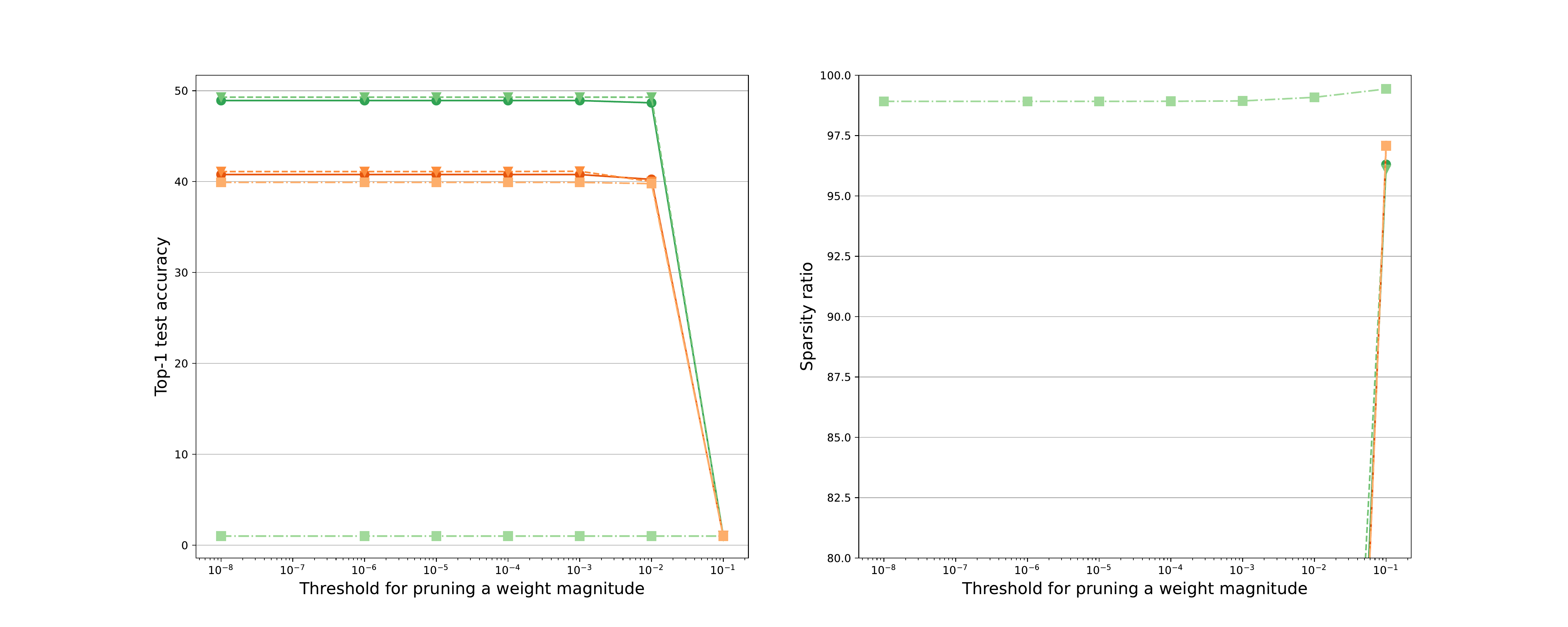}

      \caption{\label{fig:d201}%
      Accuracy and sparsity ratio of pruned DenseNet-201 on CIFAR-100 with $\mathcal{R} = \lambda\|\cdot\|_1$ (top) and $\mathcal{R} = \lambda\|\cdot\|_0$ (bottom).}
\end{figure}

\section{Conclusion}\label{sec:conclusion}

SR2 is a new stochastic proximal method for training DNNs with nonsmooth, potentially nonconvex regularizers. SR2 relies on an adaptive quadratic regularization framework that does not automatically accept every step during the training to ensure a decrease in the objective. We establish the convergence of a first-order stationarity measure to zero with a $\mathcal{O}(\epsilon^{-2})$ worst-case iteration complexity. 
Our numerical experiments show that SR2 consistently produces solutions that achieve high accuracy and sparsity levels after an unstructured pruning. Ongoing research is focusing on incorporating a momentum term, a preconditioner, and second-order information to accelerate the convergence and attain higher accuracy.


\acks{This work was supported by NSERC Alliance grant 544900- 19 in collaboration with Huawei-Canada, the Canada Excellence Research Chair in ``Data Science for Real-time Decision-making'', and Cornell Tech.
}

\bibliography{example_paper.bib}











\end{document}